\documentclass[letterpaper, 10 pt, conference]{ieeeconf}  
\usepackage{graphicx}   
\usepackage{hyperref}
\usepackage{xcolor}
\usepackage{amsmath,amssymb, amsfonts, mathtools}
\usepackage[utf8]{inputenc}
\usepackage[english]{babel}
\usepackage{subcaption}
\usepackage{soul}

\usepackage{amsthm}

\newtheorem{theorem}{Theorem}
\newtheorem{definition}{Definition}
\newtheorem{corollary}{Corollary}[theorem]
\newtheorem{remark}{Remark}
\newtheorem{lemma}{Lemma}
\newtheorem{proposition}{Proposition}

\DeclareMathOperator{\co}{co}
\DeclareMathOperator*{\argmin}{arg\,min}
\newcommand{\norm}[1]{\left\lVert#1\right\rVert}

\mathtoolsset{showonlyrefs=true}

\IEEEoverridecommandlockouts                              
\title{\LARGE \bf
Robust Barrier Functions for a Fully Autonomous, Remotely Accessible Swarm-Robotics Testbed
}

\author{Yousef Emam$^{1}$, Paul Glotfelter$^{2}$ and Magnus Egerstedt$^{3}$
  \thanks{This research was sponsored by NSF Awards CNS-1446891.}
  \thanks{Yousef Emam, Paul Glotfelter, and Magnus Egerstedt are with the Institute for Robotics and Intelligent Machines, Georgia
    Institute of Technology, Atlanta, GA 30332, USA,
    \{yemam3, paul.glotfelter, magnus\}@gatech.edu.}%
}

\begin{document}

\maketitle
\thispagestyle{empty}
\pagestyle{empty}

\begin{abstract}
    The Robotarium, a remotely accessible swarm-robotics testbed, has provided free, open access to robotics and controls research for hundreds of users in thousands of experiments.  This high level of usage requires autonomy in the system, which mainly corresponds to constraint satisfaction in the context of users' submissions.  In other words, in case that the users' inputs to the robots may lead to collisions, these inputs must be altered to avoid these collisions automatically.  However, these alterations must be minimal so as to preserve the users' objective in the experiment. Toward this end, the system has utilized barrier functions, which admit a minimally invasive controller-synthesis procedure.  However, barrier functions are yet to be robustified with respect to unmodeled disturbances (e.g., wheel slip or packet loss) in a manner conducive to real-time synthesis.  As such, this paper formulates robust barrier functions for a general class of disturbed control-affine systems that, in turn, is key for the Robotarium to operate fully autonomously (i.e., without human supervision).  Experimental results showcase the effectiveness of this robust formulation in a long-term experiment in the Robotarium.
\end{abstract}

\section{Introduction} 
\label{sec:introduction}


The Robotarium, a remotely accessible swarm-robotics testbed, located at the Georgia Institute of Technology, provides free, open access to a large number of differential-drive robots \cite{pickem2017robotarium}, alleviating the cost, in terms of both time and money, of starting a robotics testbed. To date, hundreds of people have utilized the Robotarium for over thousands of experiments.  In turn, this high level of usage necessarily requires the automatic and continuous execution of these experiments, as manually executing this large number of experiments is infeasible.

For the system to be autonomous, it must be endowed with the capabilities to enforce constraints in the context of users' submissions. For example, the robots must avoid inter-robot collisions. However, since the primary goal of the Robotarium is to enable controls research for a wide class of users, which may include nontraditional controls researchers (e.g., biologists), these constraints must be enforced in a manner that guarantees safety but also minimally interferes with users' experiments.

Toward this end, the Robotarium has extensively utilized barrier functions, as they are amenable to controller synthesis \cite{ames2014, xu2015, xu2018, wang2016, wang2016-2, PaulNBF, glotfelter2018}.  In particular, they do not encode objectives, such as may be the case for Lyapunov or potential functions (e.g., \cite{rimon1992}), which results in a minimally invasive formulation.  Formally, barrier functions satisfy constraints by guaranteeing forward invariance of a particular set that represents the constraint (e.g., the robots must always remain in the testbed).

The Robotarium has successfully executed thousands of experiments while using barrier functions to enforce constraints \cite{pickem2017robotarium}; however, uncertainty stemming from unmodeled behavior, such as packet loss or wheel slip, poses a continuing issue by preventing fully autonomous operation (i.e., with no human supervision).  For example, in its current state, a human operator must manually flag an experiment for re-execution if collisions occur.  As such, the Robotarium has a need for a formulation of barrier functions that can account for disturbances in an efficient manner.

Prior work on barrier functions has mainly addressed smooth barrier functions, formulating the associated forward-invariance results with respect to continuous dynamical systems.  Moreover, some work has focused on robust barrier functions for uncertain systems \cite{nguyen2016optimal, gurriet2018invariance}.  However, as formulated in these works, accounting for the disturbance in real-time controller synthesis involves solving a nonlinear optimization program, making the approach too costly for the Robotarium, which must quickly synthesize controllers for large groups of robots (e.g., at $60$~Hz for $40$ robots).  Since disturbed control systems may be represented by differential inclusions, the publications \cite{PaulNBF, glotfelter2018}, which utilize differential inclusions to develop a class of nonsmooth barrier functions, relate to this work on a theoretical level.  However, these results (e.g., \cite{PaulNBF}) have not been specialized for disturbed systems.

As such, this work extends barrier functions to disturbed control systems, resulting in robust Control Barrier Functions (CBFs).  To do so, this work utilizes some results from \cite{PaulNBF} by reformulating them in the context of disturbed control systems.  Specifically, this work addresses a particular class of disturbed control affine systems for which general disturbances may be addressed with linear computational complexity.  Additionally, we show that this class of robust barrier functions admits a controller-synthesis procedure via Quadratic Program (QP), ensuring that it can be run in real time, even on resource-limited systems.  Subsequently, this work specifically examines the case in which the control system is a nonlinear differential-drive robot, for usage in the Robotarium, and specializes the robust CBF framework to this context. This paper also presents a long-term experiment that demonstrates the increased autonomy afforded by the robust formulation.

The paper is organized as follows. Section~\ref{sec:background} introduces the necessary background material and notation for the paper.  In Section~\ref{sec:barrier-functions-for-disturbed}, we model the disturbances as the convex hull of a finite set of points and derive sufficient conditions for the robustness of the forward invariance property with respect to the disturbed control system.  Furthermore, Section~\ref{sec:controller-synthesis-via} provides some controller-synthesis results with respect to disturbed control systems.  Lastly, in Sections \ref{sec:robust-collision-avoidance} and \ref{sec:experiments}, we formulate the controller-synthesis procedure for differential-drive robots and apply the proposed method to the Robotarium in a long-term experiment.  Section~\ref{sec:conclusion} concludes the paper.

\section{Background Material} 
\label{sec:background}

The main contribution of this work pertains to Control Barrier Functions (CBFs) in the context of uncertain systems. As such, the theory of CBFs and differential inclusions is introduced as the main tools to address this problem.

\subsection{Notation}
\label{subsec:notation}

The notation $\mathbb{R}_{\geq a}$ represents the set of nonnegative real numbers greater or equal to $a$.  The expression $B(x', \delta)$ denotes an open ball of radius $\delta$ centered on a point $x' \in \mathbb{R}^{n}$.  The operation $\co$ represents the convex hull of a set.  A function $\alpha : \mathbb{R} \to \mathbb{R}$ is extended class-$\mathcal{K}$ if $\alpha$ is continuous, strictly increasing, and $\alpha(0) = 0$.  A function $\beta : \mathbb{R}_{\geq 0} \times \mathbb{R}_{\geq 0} \to \mathbb{R}_{\geq 0}$ is class-$\mathcal{KL}$ if it is class-$\mathcal{K}$ in its first argument and, for each fixed $r$, $\beta(r, \cdot)$ is continuous, strictly decreasing, and $\lim_{s \to \infty} \beta(r, s) = 0$.

\subsection{Control Barrier Functions} 
\label{subsec:control-barrier-functions}

Control Barrier Functions (CBFs) are formulated with respect to control systems \cite{ames2014,xu2015,AmesBarriers}, and this work considers control-affine systems
\begin{equation}
    \label{eq:control-affine}
    \dot{x}(t) = f(x(t)) + g(x(t))u(x(t)) , x(0) = x_{0} ,
\end{equation}
where $f : \mathbb{R}^{n} \to \mathbb{R}^{n}$, $g : \mathbb{R}^{n} \to \mathbb{R}^{n \times m}$, and $u : \mathbb{R}^{n} \to \mathbb{R}^{m}$ are continuous.  These types of systems capture most robotic systems (e.g., differential-drive robots, quadrotors, autonomous vehicles) and remain amenable to controller synthesis, which is demonstrated later.  A set $\mathcal{C}$ is called forward invariant with respect to \eqref{eq:control-affine} if given a solution (potentially nonunique) to \eqref{eq:control-affine} $x : [0, t_{1}] \to \mathbb{R}^{n}$
\begin{equation}
    x_{0} \in \mathcal{C} \implies x(t) \in \mathcal{C}, \forall t \in [0, t_{1}] .
\end{equation}

Barrier functions ensure forward invariance of a particular set that typically represents a constraint in a robotic system, such as collision avoidance or connectivity maintenance.   Specifically, a barrier function is a continuously differentiable function $h : \mathbb{R}^{n} \to \mathbb{R}$ (sometimes referred to as a candidate barrier function), and the so-called safe set $\mathcal{C} \subset \mathbb{R}^{n}$ is defined as the super-zero level set to $h$
\begin{equation}
    \mathcal{C} = \{x' \in \mathbb{R} : h(x') \geq 0\} .
\end{equation}
Now, the goal becomes to ensure the forward set invariance of $\mathcal{C}$, which can be done equivalently by guaranteeing positivity of $h$ along trajectories.

Positivity can be shown if there exists a locally Lipschitz extended class-$\mathcal{K}$ function $\alpha : \mathbb{R} \to \mathbb{R}$ and a continuous function $u : \mathbb{R}^{n} \to \mathbb{R}^{m}$ such that
\begin{equation}
    \label{eq:barrier-certificate-reg}
    \nabla h(x')^{\top}(f(x') + g(x')u(x')) \geq -\alpha(h(x')), \forall x' \in \mathbb{R}^{n}.
\end{equation}
Then, $h$ is called a valid CBF for \eqref{eq:control-affine} \cite{AmesBarriers}.  Note that \eqref{eq:barrier-certificate-reg} does not explicitly account for any uncertainty in the system.  As such, real-world disturbances (e.g., packet loss or wheel slip) can cause the system to violate the constraint.  Toward resolving this issue, the next section formulates an analogous result in the context of uncertain control systems.

\subsection{Differential Inclusions}
\label{subsec:differential-inclusions}

Differential inclusions are a generalization of differential equations that have been used to represent a variety of problems including perturbed Ordinary Differential Equations (ODEs) as in \eqref{eq:control-affine} and discontinuous dynamical systems \cite{PaulNBF}.  Uncertain or disturbed systems, in the context of this paper, fall into a particular class of differential inclusions.  As such, this section presents the high-level theory of differential inclusions, and later sections formulate the disturbed system that this work considers.

In general, differential inclusions are formulated as 
\begin{equation} 
    \label{eq:diffInc}
    \dot{x}(t) \in F(x(t)), x(0) = x_{0} ,
\end{equation}
where $F : \mathbb{R}^{n} \to 2^{\mathbb{R}^{n}}$ is an upper semi-continuous set-valued map that takes nonempty, convex, and compact values.  A set-valued map is called upper semi-continuous if for every $\epsilon > 0$, $x' \in \mathbb{R}^{n}$ there exists $\delta > 0$ such that 
\begin{equation}
    F(y) \subset F(x') + B(0, \epsilon), \forall y \in B(x', \delta) .
\end{equation}
Note that the term $F(\cdot) + B(0, \epsilon)$ is meant to be taken as a set-valued addition.  That is,
\begin{equation}
    F(\cdot) + B(0, \epsilon) = \{v + e : v \in F(\cdot), \|e\| \leq \epsilon\} .
\end{equation}

Under the given assumptions for $F$ in \eqref{eq:diffInc}, existence of a particular type of solution, a Carath\'eodory solution, may be guaranteed \cite{cortes2008}. A Carath\'eodory solution to \eqref{eq:diffInc} is an absolutely continuous trajectory $x$ : $[0, t_1] \to \mathbb{R}^n$ such that $\dot{x}(t) \in F(x(t))$ almost everywhere on $[0, t_{1}]$ and $x(0) = x_{0}$.  In this case, uniqueness can by no means be guaranteed.  For a more comprehensive survey of differential inclusions, see \cite{cortes2008discontinuous}.


Forward invariance in the context of differential inclusions typically admits two standard definitions, stemming from the nonuniqueness of solutions \cite{cortes2008discontinuous}.  Weak invariance insists that, given an initial condition in $\mathcal{C}$, at least one solution remains in the set for all time.  Strong invariance requests that, given an initial condition in $\mathcal{C}$, all solutions stay in the set for all time.  This paper provides results for strong invariance and simply refers to this quality as invariance.


\subsection{Barrier Functions for Differential Inclusions} 
\label{subsec:barrier-functions-for-differential-inclusions}

The work in \cite{PaulNBF} generalizes the result in \eqref{eq:barrier-certificate-reg} to nonsmooth barrier functions and differential inclusions.  However, in the case that the barrier function is smooth, the same result applies to a system described by a differential inclusion.  As such, this result becomes useful for this work, and it is subsequently stated in a form that has been modified to fit the terminology of this work.  In the next section, we specialize this result for the purpose of robust control and apply it to validate CBFs for disturbed control systems.  Note that the results here are originally phrased for uncontrolled systems in \cite{PaulNBF}; however, for brevity, we still refer to barrier functions as CBFs, as the same results hold by considering the closed-loop system.

\begin{definition}{\cite[Definition~4]{PaulNBF}} 
    A locally Lipschitz function $h : \mathbb{R}^{n} \to \mathbb{R}$ is a valid \emph{Control Barrier Function (CBF)} for \eqref{eq:diffInc} if and only if $x_{0} \in \mathcal{C}$ implies that there exists a class-$\mathcal{KL}$ function $\beta : \mathbb{R}_{\geq 0} \times \mathbb{R}_{\geq 0} \to \mathbb{R}_{\geq 0}$ such that 
    \begin{equation}
        h(x(t)) \geq \beta(h(x_{0}), t), \forall t \in [0, t_{1}] ,
    \end{equation}
    for every Carath\'eodory solution $x : [0, t_{1}] \to \mathbb{R}^{n}$ starting from $x_{0}$.
\end{definition}

\begin{theorem}{\cite[Theorem~1]{PaulNBF}}
    \label{thm:valid-cbf}
    Let $h : \mathbb{R}^{n} \to \mathbb{R}$ be a continuously differentiable function.  If there exists a locally Lipschitz extended class-$\mathcal{K}$ function $\alpha: \mathbb{R} \to \mathbb{R}$ such that 
    \begin{equation}
        \min \nabla h(x')^{\top}F(x') \geq - \alpha(h(x')), \forall x' \in \mathbb{R}^n,
    \end{equation}
    then $h$ is a valid CBF for \eqref{eq:diffInc}.
\end{theorem}
\begin{remark}
    In \cite{PaulNBF}, the gradient $\nabla h$ is replaced with the Clarke generalized gradient $\partial_{c} h : \mathbb{R}^{n} \to 2^{\mathbb{R}^{n}}$ (see \cite{clarke1990}) and $h$ is only assumed to be locally Lipschitz continuous.  However, in the case that $h$ is continuously differentiable, these two objects are equivalent.  Hence, Theorem~\ref{thm:valid-cbf} is equivalent to \cite[Theorem~1]{PaulNBF} in the context of this work.
\end{remark}

\section{Barrier Functions for Disturbed Dynamical Systems} 
\label{sec:barrier-functions-for-disturbed}

This section discusses a formulation of CBFs that can account for a large class of disturbances in an efficient manner.  In particular, this section contains an extension of Theorem~\ref{thm:valid-cbf} to disturbed control-affine systems.  Then, a class of systems with disturbances described by convex sets is presented along with efficient methods for validating the associated CBFs.

In this paper, we will focus on disturbed control-affine systems that can be modelled through the following differential inclusion
\begin{equation} 
    \label{eq:control-affine-disturbed}
    \dot{x}(t) \in f(x(t)) + g(x(t))u(x(t)) + D(x(t)), x(0) = x_{0} ,
\end{equation}
where $D : \mathbb{R}^{n} \to 2^{\mathbb{R}^{n}}$ (the disturbance) is an upper semi-continuous set-valued map that takes nonempty, convex, and compact values; and $f$, $g$, $u$ are as in \eqref{eq:control-affine}.  The assumption on the convexity of $D$ may be seen as restrictive, but we later show that this assumption is actually highly nonrestrictive.  Following from Theorem~\ref{thm:valid-cbf}, we can state a corollary that specifically targets \eqref{eq:control-affine-disturbed}.

\begin{corollary}
    \label{cor:disturbed-cbf}
    Let $h : \mathbb{R}^n \to \mathbb{R}$ be a continuously differentiable function.  If there exists a continuous function $u : \mathbb{R}^{n} \to \mathbb{R}^{m}$ and a locally Lipschitz extended class-$\mathcal{K}$ function $\alpha: \mathbb{R} \to \mathbb{R}$ such that 
    \begin{align}
        \begin{split}
            & \min \nabla h(x')^{\top}(f(x') + g(x')u(x') + D(x')) \geq \\
            & - \alpha(h(x')), \forall x' \in \mathbb{R}^n,
        \end{split}
    \end{align}
    then $h$ is a valid CBF for \eqref{eq:control-affine-disturbed}.
\end{corollary}

By using the corollary above, the main contribution of this paper utilizes the properties of convex hulls to generate a control law robust to set-valued disturbances (i.e., as in \eqref{eq:control-affine-disturbed}).  Specifically, we assume that a convex hull of $p > 0$ continuous functions $\psi_{i} : \mathbb{R}^{n} \to \mathbb{R}^{n}$, $i \in \{1, \hdots, p\}$ captures the disturbance.  That is, 
\begin{equation} 
    \label{eq:defPsi}
    D(x') = \co\Psi(x') = \co\{\psi_1(x')\ldots\psi_p(x')\}, \forall x' \in \mathbb{R}^{n} .
\end{equation}

In this case, the differential inclusion in \eqref{eq:control-affine-disturbed} can be re-written as
\begin{equation} 
    \label{eq:control-affine-disturbed-co}
    \dot{x}(t) \in f(x(t)) + g(x(t))u(x(t)) + \co \Psi(x(t)) ,
\end{equation}
and the first main result of this paper pertains to validating CBFs with respect to \eqref{eq:control-affine-disturbed-co}.
\begin{lemma}
    \label{lem:usc-disturbance}
    Let $\psi_{i} : \mathbb{R}^{n} \to \mathbb{R}^{n}$, $i \in \{1, \hdots, p\}$ be a set of $p > 0$ continuous functions.  Then, $D : \mathbb{R}^{n} \to 2^{\mathbb{R}^{n}}$ defined as 
    \begin{equation}
        D(x') = \co \Psi(x') = \co \{\psi_{i}(x') : i \in \{1, \hdots, p\}\}, \forall x' \in \mathbb{R}^{n}
    \end{equation}
    is an upper semi-continuous set-valued map that takes compact, nonempty, and convex values.
\end{lemma}
\begin{proof}
    Let $x' \in \mathbb{R}^{n}$.  Then, by definition, $D(x')$ is convex and nonempty, since $p > 0$.  Moreover, $D(x')$ is compact, since it is the convex hull of a finite number of points.  
    
    Now, it remains to show upper semi-continuity.  Let $\epsilon > 0$.  Because each $\psi_{i}$ is continuous, there exists corresponding $\delta_{i} > 0$ such that 
    \begin{equation}
        \|\psi_{i}(y) - \psi_{i}(x')\| \leq \epsilon, \forall y \in B(x', \delta_{i}) ,
    \end{equation}
    meaning that 
    \begin{equation}
        \psi_{i}(y) \in \psi_{i}(x') + B(0, \epsilon), \forall y \in B(x', \delta_{i}) .
    \end{equation}
    Set 
    \begin{equation}
        \delta = \min_{i} \delta_{i} ,
    \end{equation}
    which satisfies $\delta > 0$, because $p$ is finite.
     
    Since
    \begin{equation}
    \label{eq:disturbance}
        D(\cdot) = \co \{\psi_{i}(\cdot) : i \in \{1, \hdots, p\}\} , 
    \end{equation}
    it now suffices to show that 
    \begin{align}
        \Psi(y) \subset \Psi(x') + B(), \epsilon), \forall y \in B(x', \delta) .
    \end{align} 
    Let $i \in \{1, \hdots, p\}$ and consider $\psi_{i}(y)$.  By choice of $\delta$, 
    \begin{equation}
        \psi_{i}(y) \in \psi_{i}(x') + B(0, \epsilon) ,
    \end{equation}
    as such 
    \begin{equation}
        \psi_{i}(y) \in \Psi(x') + B(0, \epsilon) .
    \end{equation}
    Accordingly, $D$ is upper semi-continuous.
\end{proof}

\begin{theorem} 
    \label{mainTheorem}
    Let $h : \mathbb{R}^{n} \to \mathbb{R}$ be a continuously differentiable function.  Let $\psi_{i} : \mathbb{R}^{n} \to \mathbb{R}^{n}$, $i \in \{1, \hdots, p\}$ be a set of $p > 0$ continuous functions, and define the disturbance $D : \mathbb{R}^{n} \to 2^{\mathbb{R}^{n}}$ as 
    \begin{equation}
        D(x') = \co \Psi(x') = \co\{\psi_1(x')\ldots\psi_p(x')\}, \forall x' \in \mathbb{R}^{n} .
    \end{equation}
    If there exists a continuous function $u : \mathbb{R}^{n} \to \mathbb{R}^{m}$ and a locally Lipschitz extended class-$\mathcal{K}$ function $\alpha : \mathbb{R} \to \mathbb{R}$ such that 
    \begin{equation} 
        \label{eq:mainTheorem}
            \begin{split}
                & \nabla h(x')^{\top}(f(x') + g(x')u(x')) \geq \\
                & -\alpha(h(x')) - \min \nabla h(x')^{\top} \Psi(x'), \forall x' \in \mathbb{R}^{n} ,
            \end{split}
    \end{equation}
    then $h$ is a valid CBF for \eqref{eq:control-affine-disturbed}.
\end{theorem}
\begin{proof}
    \label{thm:valid-cbf-disturbed}
    We begin by substituting the definition of $D(x')$ from \eqref{eq:defPsi} into the result from Corollary~\ref{cor:disturbed-cbf}.  In particular, it must be shown that
    \begin{align}
        \begin{split}
            & \min \nabla h(x')^{\top}(f(x') + g(x')u(x') + \co \Psi(x')) \geq \\
            & - \alpha(h(x')), \forall x' \in \mathbb{R}^{n} .
        \end{split}
    \end{align}
    By Lemma~\ref{lem:usc-disturbance}, $\co \Psi$ is an upper semi-continuous set-valued map that takes nonempty, convex, and compact values, so the results of Corollary~\ref{cor:disturbed-cbf} may be applied.
    
    Note that, for any $x' \in \mathbb{R}^{n}$, the condition above is equivalent to
    \begin{align} 
        \begin{split}
            \label{eq:ineq1}
            & \nabla h(x')^{\top}(f(x') + g(x')u(x')) \geq \\
            & -\alpha(h(x')) - \min \nabla h(x')^{\top} \co \Psi(x') .
        \end{split}
    \end{align}
    We can then take advantage of the properties of the convex hull (see \cite[Lemma~3]{PaulNBF}) through the following equality
    \begin{equation} 
        \label{eq:minCo}
        \min \nabla h(x')^{\top} \text{co} \Psi(x') = \min \nabla h(x')^{\top} \Psi(x') .
    \end{equation}
    Thus, by substituting \eqref{eq:minCo} into \eqref{eq:ineq1}, we obtain \eqref{eq:mainTheorem}.
\end{proof}
\begin{remark}
    \label{rem:min-continuous}
    Note that the function 
    $$
        x' \mapsto \min \nabla h(x')^{\top} \co \Psi(x') = \min_{i = \{1, \hdots, p\}} \nabla h(x')^{\top}\psi_{i}(x') 
    $$
    is continuous, since it is a minimum of continuous functions.
\end{remark}

An interesting aspect of Theorem~\ref{thm:valid-cbf-disturbed} is that only the extreme points of $\co \Psi(\cdot)$ must be evaluated (i.e., each $\psi_{i}(\cdot)$).  This evaluation remains equivalent to checking every disturbance in $\co \Psi(\cdot)$.  Thus, infinitely many disturbances are addressed by checking a finite number of points.  As such, Theorem~\ref{thm:valid-cbf-disturbed} can be directly used for synthesizing controllers robust to disturbance, because the computational cost of $\min \nabla h(x')^{\top} \Psi(x')$ is linear with respect to the size of the set $\Psi(x')$.

Notably, in Theorem~\ref{thm:valid-cbf-disturbed}, it may not be the case that:
\begin{equation} 
    \label{eq:addZero}
    \mathbf{0}_{n} \in D(x'), \forall x' \in \mathbb{R}^{n} ,
\end{equation}
and therefore undisturbed trajectories may be unmodeled. However, if \eqref{eq:addZero} holds, then the undisturbed solutions that start within $\mathcal{C}$ also remain within $\mathcal{C}$.  Additionally, Theorem~\ref{thm:valid-cbf-disturbed} readily generalizes to the case where $D$ is the union of the convex hulls of a finite number of function-valued points, as shown by the following proposition.

\begin{proposition}
    \label{prop:union-of-convex-hulls}
    Define 
    \begin{equation} 
        \label{eq:Union}
        D(x') = \bigcup_{i=1}^{q} \co(\Psi_{i}(x')), \forall x' \in \mathbb{R}^{n}
    \end{equation}
    such that
    \begin{equation}
        \Psi_{i}(x') = \{\psi^i_{1}(x'), \hdots, \psi^{i}_{p_{i}}(x')\}, \forall i \in \{1, \hdots, q\}, 
    \end{equation}
    where each $\psi^{i}_{j}(\cdot)$ is continuous. Let $h : \mathbb{R}^{n} \to \mathbb{R}$ be a continuously differentiable function.  If there exists a continuous function $u : \mathbb{R}^{n} \to \mathbb{R}^{m}$ and a locally Lipschitz extended class-$\mathcal{K}$ function $\alpha : \mathbb{R} \to \mathbb{R}$ such that 
    \begin{align} 
        \begin{split}
            & \nabla h(x')^{\top}(f(x') + g(x')u(x')) \geq \\
            & -\alpha(h(x')) - \min \nabla h(x')^{\top} \Psi_i(x'), \forall i \in \{1, \hdots, q\} ,
        \end{split}
    \end{align}
    then $h$ is a valid CBF for \eqref{eq:control-affine-disturbed-co}.
\end{proposition}
\begin{proof} 
    This proposition is directly obtained by applying Theorem~\ref{thm:valid-cbf-disturbed} $q$ times.  Given that inequality \eqref{eq:mainTheorem} holds for each $D_i(x') = \co \Psi_i(x')$ then it follows that it holds for $ D(x') = \bigcup \limits_{i=1}^{q} D_i(x')$, since
\begin{gather}
    \min(\bigcup\limits_{i=1}^{q}Y_i) = \min(\{\min(Y_1),\ldots,\min(Y_q)\})  ,
\end{gather}
where each $Y_i$ is a compact set containing an arbitrary number of scalar values. 
\end{proof}
\begin{remark}
    \label{rem:magic-method}
    The line of reasoning in the proof above illustrates the fact that $h$ is a valid CBF for 
    \begin{equation}
        D(x') = \bigcup_{i=1}^{q} \co \Psi_{i}(x'), \forall x' \in \mathbb{R}^{n}
    \end{equation}
    if and only if it is also valid for
    \begin{equation}
        D(x') = \co \bigcup\limits_{i=1}^{q} \Psi_{i}(x'), \forall x' \in \mathbb{R}^{n}
    \end{equation}
\end{remark}

The nonconvexity of $D$, as given in Proposition~\ref{prop:union-of-convex-hulls}, may appear to pose a problem, as the sufficient conditions for the existence of solutions to a differential inclusion requires that the set-valued map takes convex values (see \eqref{eq:diffInc}).  However, as noted by Remark~\ref{rem:magic-method}, $D$ may be equivalently defined as 
\begin{equation}
    D(x') = \co \bigcup\limits_{i=1}^{q} \Psi_{i}(x'), \forall x' \in \mathbb{R}^{n} ,
\end{equation}
which is indeed convex.

Describing disturbances as a finite union of convex sets encodes a very wide class of disturbances.  However, Remark~\ref{rem:magic-method} implies that considering this nonconvex disturbance is actually completely equivalent to considering the convex hull of the disturbance.  As such, this result indicates that utilizing convex hulls to approximate a disturbance equivalently addresses a wide class of nonconvex disturbances.

The simplicity of the condition presented in Theorem~\ref{thm:valid-cbf-disturbed} allows for the synthesis of a controller that satisfies \eqref{eq:mainTheorem} for a given system in real time, since only a finite number of points need to be evaluated.  In the next section, we present a QP for the purposes of controller synthesis, which we apply to the Robotarium as presented in Section~\ref{sec:robust-collision-avoidance}.

\section{Controller Synthesis via Quadratic Program}
\label{sec:controller-synthesis-via}

Many of the Robotarium's users do not incorporate collision avoidance into their algorithms (e.g. \cite{Hatanaka2009,Park2016}). This may be due to lack of knowledge, since access to the Robotarium is not restricted to roboticists and controls researchers, and/or the fact that many of these algorithms are designed for higher level objectives. To mitigate this issue, we minimally alter the robots' control inputs so as to avoid collisions through the application of the framework presented below. The choice of using CBFs instead of its alternatives (e.g. potential functions, Lyapunov functions) stems from the fact that CBFs do not encode any underlying assumptions about the user's objective, making them the ideal choice for this application. Moreover, since differential-drive robots are susceptible to issues such as wheel slip and network latency, by incorporating noise into the CBFs, we render this framework robust to these various risks of collision. In this section, we will introduce the QP used to synthesize the minimally invasive controller with respect to a nominal input.  

The dynamics of the robots are modeled as follows, 
\begin{gather} 
    \label{eq:diffIncU}
    \dot{x}(t) \in f(x(t)) + g(x(t))(u(x(t)) + \co \Psi) , 
\end{gather}
where $\Psi = \{\psi_1, \hdots, \psi_p\} \subset \mathbb{R}^{n}$, for some $p > 0$.

Note that the disturbance appears in the control input. Referring to \eqref{eq:control-affine-disturbed}, we can re-write the latter as
\begin{gather} 
    \label{eq:diffIncU2}
    \dot{x}(t) \in f(x(t)) + g(x(t))u(x(t)) + D(x(t)) ,
\end{gather}
where $D(x(t)) = \co g(x(t))\Psi$, via properties of the convex hull operation in multiplication.  Note that each function $x' \mapsto g(x')\psi_{i}$ is a continuous function, because $g$ is assumed to be continuous and $\Psi$ is a collection of constant vectors.  The reason that we choose to model the disturbance as such is because various issues (e.g., wheel slip) can be modelled as offsets to the input motor velocities.
\begin{proposition}
    \label{prop:controller-synthesis}
    Let $h : \mathbb{R}^{n} \to \mathbb{R}$ be a continuously differentiable function, and let $\psi_{i} \in \mathbb{R}^{n}$, $i \in \{1, \hdots, p\}$ be a set of $p > 0$ constant vectors.  Define
    \begin{equation}
        D(x') = \co g(x')\Psi = \co g(x')\{\psi_{i} : i \in \{1, \hdots, p\}\}, 
    \end{equation}
    for all $x' \in \mathbb{R}^{n}$.  If $u^{*} : \mathbb{R}^{n} \to \mathbb{R}^{m}$ defined as  
    \begin{align}
        \label{QP}
        & u^{*}(x') = \argmin_{u \in \mathbb{R}^{m}} \|u_{\text{nom}}(x') - u\|^2 \\
        & \text{s.t. } \nabla h(x')^{\top}(f(x') + g(x')u(x')) \geq -\alpha(h(x')) \\
        & \qquad - \min \nabla h(x')^{\top} g(x')\Psi(x')
    \end{align}
    is continuous, then $h$ is valid CBF for \eqref{eq:control-affine-disturbed}.
\end{proposition}

Note that since each $x' \mapsto g(x')\psi_{i}(x')$ and $\nabla h$ are continuous, so is $x' \mapsto \min \nabla h(x')^{\top} g(x')\Psi(x')$ (see Remark~\ref{rem:min-continuous}).  Proposition~\ref{prop:controller-synthesis} assumes continuity of $u^{*}$; however, under continuity of $x' \mapsto \min \nabla h(x')^{\top} g(x')\Psi$, certain constraint-qualification assumptions lead to continuity of $u^{*}$.  As such, continuity of $u^{*}$ may be guaranteed under one of these conditions (e.g., see \cite{morris2013}).  Assuming we are given a nominal control $u_{\text{nom}}$ and a CBF $h$ for \eqref{eq:robDynamics}, this QP has the objective of minimally changing $u_{\text{nom}}$ such that $h$ is a valid CBF for \eqref{eq:robDynamics}.  In terms of application, we assume that the QP above is solved at each point in time $t \in [0, t_1]$. 

The QP contained in Proposition~\ref{prop:controller-synthesis} is strongly convex.  As such, most solvers admit a polynomial-time solution in the number of decision variables (i.e., control inputs) $O(m^{3})$.  Accounting for the disturbances only requires the computation of $\min \nabla h(\cdot)^{\top}g(\cdot)\Psi(\cdot)$, which is linear in the number of function from which $\Psi$ is composed and has a runtime of $O(p)$.  Thus, the overall runtime of the synthesis procedure in Proposition~\ref{prop:controller-synthesis} is $O(m^{3} + p)$.

\section{Robust Collision Avoidance for the Robotarium}  \label{sec:robust-collision-avoidance}

\begin{figure}[t]
    \centering
    \includegraphics[width=0.30\textwidth]{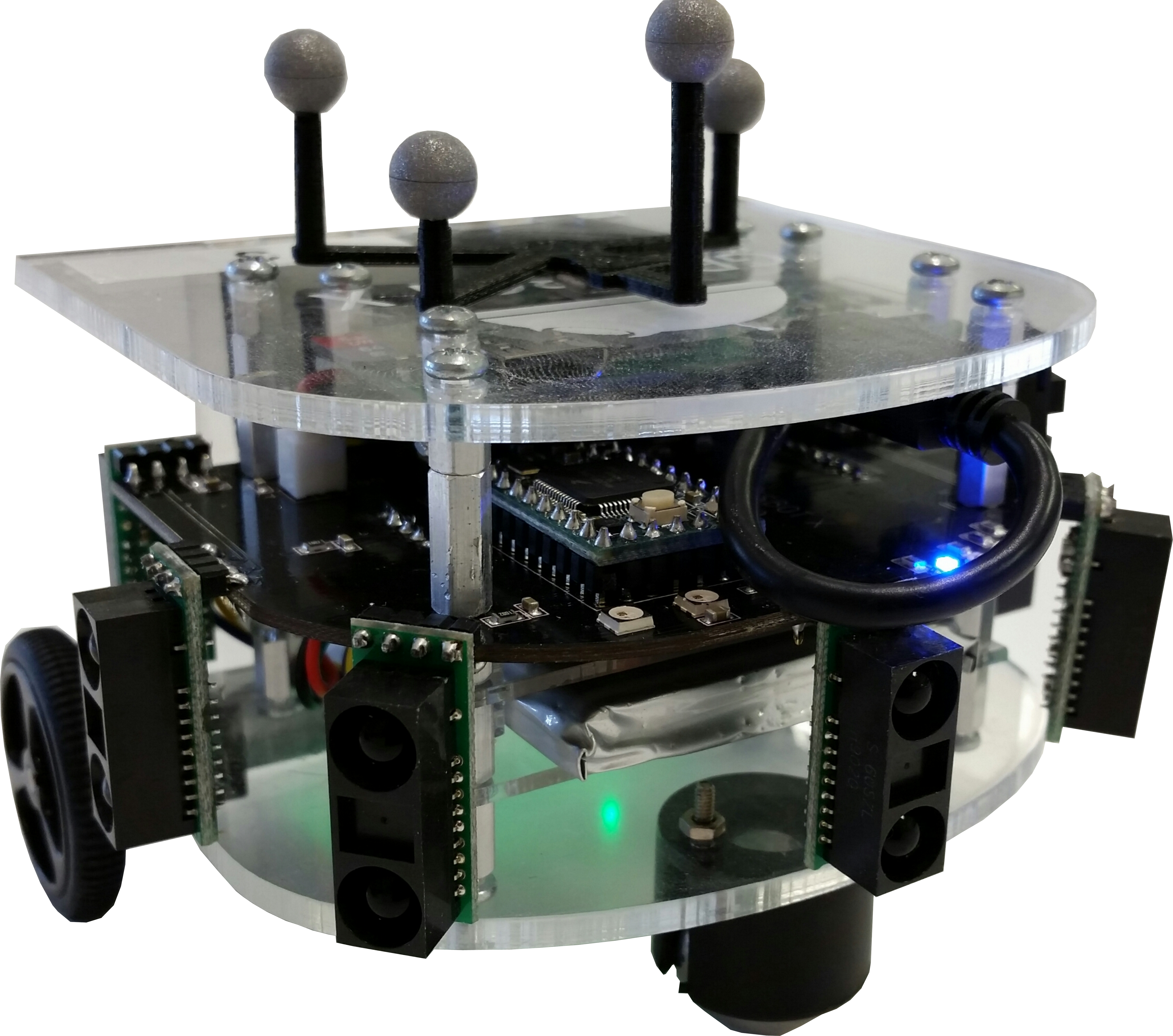}
    \caption{A picture of the GRITSbot, the differential-drive robot used by the Robotarium.}
    \label{fig:GRITSBOT}
\end{figure}

In this section, we specialize the robust CBFs detailed in Section~\ref{sec:barrier-functions-for-disturbed} and the controller-synthesis framework in Section~\ref{sec:controller-synthesis-via} for differential-drive robots for use in the Robotarium.  Specifically, the controller-synthesis procedure is formulated with respect to collision-avoidance CBFs.  For the sake of clarity, we omit the explicit dependence on time for brevity. 

Consider $N$ differential-drive robots, where each robot has state $x_{i} \in \mathbb{R}^3$ composed of its global position in the plane and heading
\begin{gather} 
    \label{eq:robDynamics}
    x_{i} \coloneqq 
    \begin{bmatrix}
     x_{i, 1} ~ x_{i, 2} ~ \theta_{i}
     \end{bmatrix}^{\top} .
\end{gather}
Each robot has dynamics
\begin{gather}
    \dot{x}_{i} \in 
    \begin{bmatrix}
        \cos \theta_{i} & 0 \\ 
        \sin \theta_{i} & 0 \\ 
        0 & 1
    \end{bmatrix}
    G(u_{i}(x_{i})+\co\Psi) ,\\
\end{gather}

where $$ \Psi = 
    \{
    \begin{bmatrix}
    \psi \\
    \psi
    \end{bmatrix},
    \begin{bmatrix}
    \psi \\
    -\psi
    \end{bmatrix},
        \begin{bmatrix}
    -\psi \\
    \psi
    \end{bmatrix},
        \begin{bmatrix}
    -\psi \\
    -\psi
    \end{bmatrix}
    \}, \psi \in \mathbb{R}_+.$$
The matrix $G$ is given by
\begin{gather}
    G =
    \begin{bmatrix}
         r/2 & r/2 \\
         -r/l_{b} &  r/l_{b}
    \end{bmatrix} ,
\end{gather}
where $r$ and $l_{b}$ are the wheel radius and the robot's base length respectively, and the differential-drive input is given by
\begin{gather}
    u_{i} =
    \begin{bmatrix}
         \omega^{R}_{i},~ \omega^{L}_{i} 
    \end{bmatrix}^{\top} ,
\end{gather}
where $\omega^{R}_{i}$ and $\omega^{L}_{i}$ are the right and left wheels' angular velocities, respectively. This geometric model, illustrated in Figure~\ref{fig:dd_diagram}, is representative of the control methodology of the Robotarium's robots shown in Figure~\ref{fig:GRITSBOT}, for which the control inputs are velocity commands to the DC motors of the robot.  As such, this formulation allows one to account for the actuation limits during controller synthesis.  

\begin{figure}[tb]
    \centering
    \includegraphics[width=0.40\textwidth]{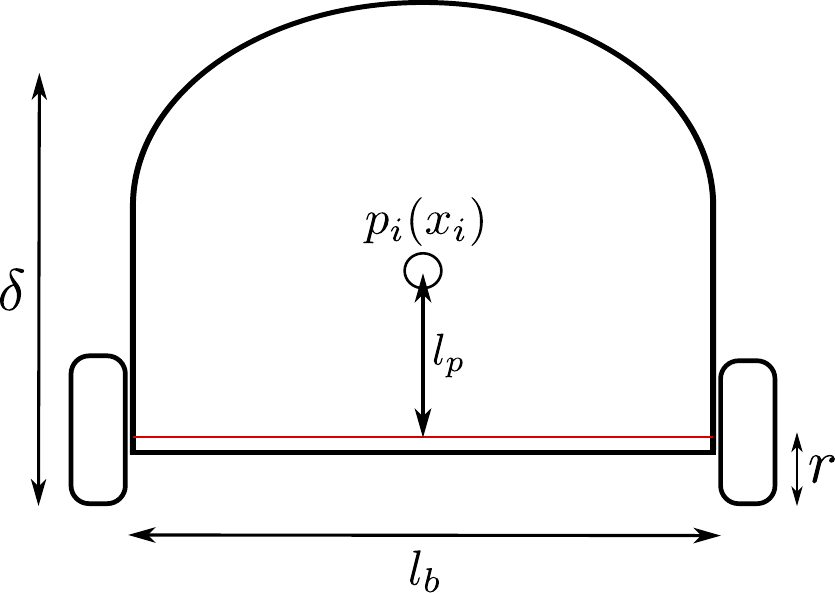}
    \caption{Diagram of the GRITSbot.  The symbol $l_b$ denotes the base length, $r$ the wheel radius, $l_p$ the projection distance, $\delta$ the diameter of the GRITSbot, and $p$ the center of the robot.  Note that $l_p$ projects the wheel-axle's center-line, denoted by the red-line, to the center-line of the robot.}
    \label{fig:dd_diagram}
\end{figure}

\begin{figure*}[tb]
\centering
\begin{minipage}{0.33\textwidth}
\includegraphics[width=\textwidth]{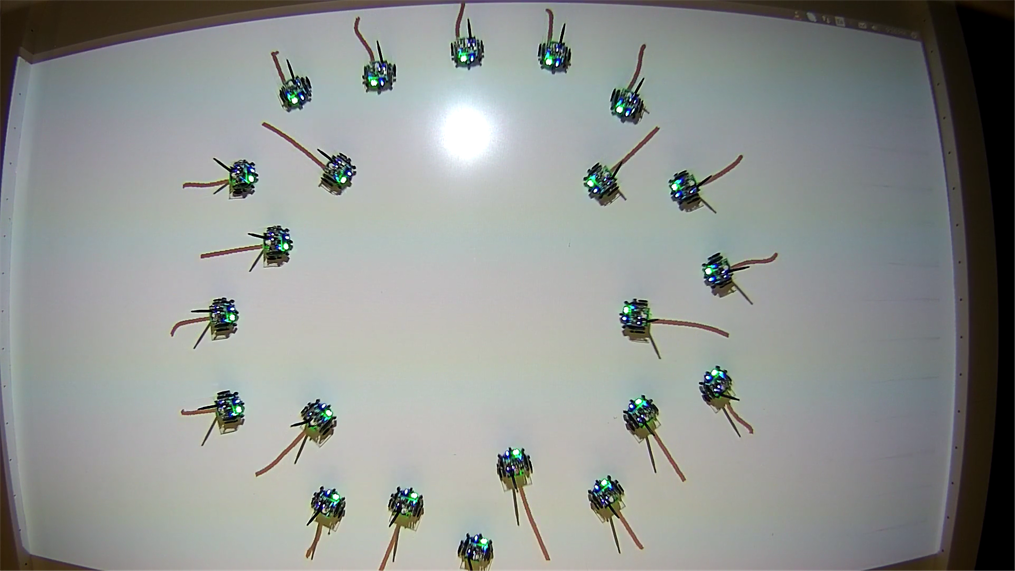}
\end{minipage}~%
\begin{minipage}{0.33\textwidth}
\includegraphics[width=\textwidth]{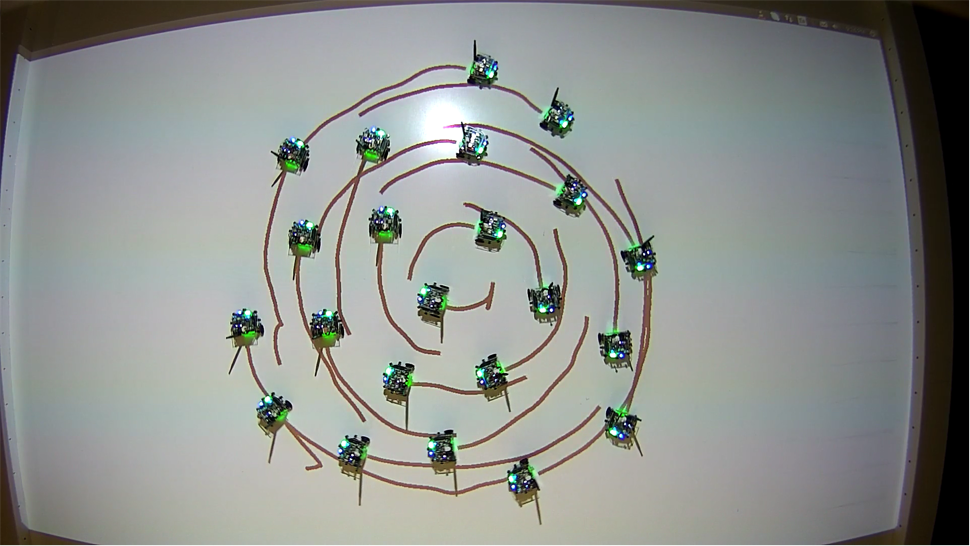}
\end{minipage}~%
\begin{minipage}{0.33\textwidth}
\includegraphics[width=\textwidth]{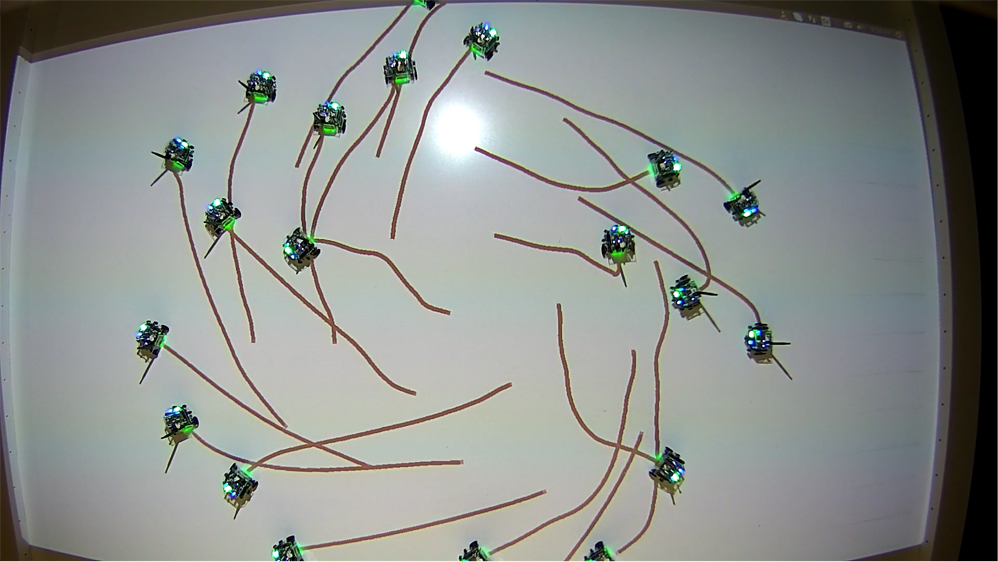}
\end{minipage}~%
\caption{In the Robotarium, a group of 22 GRITSbots complete an iteration of the repeated experiment detailed in Section~\ref{sec:experiments}.  The robots are initially arranged on a circle (left) and attempt to traverse to the opposite side (right).  The Robotarium utilizes the robust-CBF-based controller synthesis (see \eqref{eq:QP2}) to prevent collisions (middle) and ensure that each robot reaches the opposite side of the circle (right).}
\label{fig:xbarrier}
\end{figure*}

The Robotarium's differential-drive robots' wheels are located toward the back of the robot.  To permit the formulation of a collision-avoidance constraint from the centroid of the robot, we introduce the following output of the state.  This technique considers a point at a distance $l_{p} \geq 0$ ahead of the robot and orthogonal to the wheel axis.  In deployment, this parameter is chosen to lie on the centroid of the robot as shown in Figure~\ref{fig:dd_diagram}.  


In particular, define the output (i.e., the point ahead of the robot) $p_{i} : \mathbb{R}^{3} \to \mathbb{R}^{2}$ as
\begin{align}
    p_{i}(x_{i}) = 
    \begin{bmatrix}
        x_{i1} \\
        x_{i2}
    \end{bmatrix}
    + l_{p}
    \begin{bmatrix}
        \cos{\theta_i} \\
        \sin{\theta_i} 
    \end{bmatrix} .
\end{align}
Differentiating $p_{i}$ along the differential-drive dynamics yields that
\begin{gather} 
    \label{eq:pi_dot}
    \dot{p}_{i}(x_{i}) = g_i(x_{i})u_{i}(x_{i}) ,
\end{gather}
where
\begin{equation}
    g_{i}(x_{i}) = R_{i}(\theta_{i})LG,~ R_{i}(\theta_{i}) = 
    \begin{bmatrix}
        \cos \theta_{i} & -\sin \theta_{i} \\
        \sin \theta_{i} & \cos \theta_{i} 
    \end{bmatrix} 
\end{equation}
and
\begin{equation}
L = 
    \begin{bmatrix}
        1 & 0 \\ 
        0 & l_{p}
    \end{bmatrix} .
\end{equation}
Note that both $R_{i}(\cdot)$ and $L$ are always invertible, so choosing $p_{i}$ in this manner yields an invertible mapping between $\dot{p}_{i}$ and $u_{i}$.  For later convenience, define the ensemble variables
\begin{gather}
    u = 
    \begin{bmatrix}
        u^{\top}_1 & \ldots & u_N^{\top}
    \end{bmatrix}^{\top},~ %
    x = 
    \begin{bmatrix}
    x^{\top}_1 & \ldots & x_N^{\top}
    \end{bmatrix}^{\top} \\
   p = 
    \begin{bmatrix}
        p^{\top}_1 & \ldots & p_N^{\top}
    \end{bmatrix}^{\top}  
\end{gather}

Using the fact that $l_{p}$ is chosen so that $p_{i}$ is at the centroid of each robot, the following CBF encodes a collision-avoidance constraint between robots $i$ and $j$
\begin{equation} 
    \label{eq:collision-avoidance-cbf}
    h_{ij}(p(x)) = \norm{p_{i}(x_{i}) - p_{j}(x_{j})}^2 - \delta^{2} ,
\end{equation}
where $\delta > 0$ denotes the diameter of the robot (see Figure~\ref{fig:dd_diagram}).  Note that
\begin{equation}
    \nabla_{p_{i}} h_{ij}(p(x)) = (p_{i}(x_{i}) - p_{j}(x_{i})) = - \nabla_{p_{j}} h_{ij}(p(x)) . 
\end{equation}
`
The CBF in \eqref{eq:collision-avoidance-cbf} models each robot as a circle of diameter $\delta$. Again, note that the center of each robot's wheel axle is shifted from the center of the circle.  Conveniently, this issue can be easily mitigated by setting the look-ahead distance $l_p$ to map the point $p_i$ to the center of the robot. 

The barrier certificate that needs to be satisfied for each pair $(i, j)$ of robots is obtained via Proposition~\ref{prop:controller-synthesis}
\begin{align}
    \label{eq:barrier-certificate}
    & \nabla_{p_i} h_{ij}(p(x))^{\top}g_i(x)u_{i} + \nabla_{p_j} h_{ij}(p(x))^{\top}g_j(x)u_{j} \geq  \\
    & - \gamma h_{ij}(x)^3 \\
    & - \min{(\nabla_{p_i} h_{ij}(p(x))^{\top}g_i(x) + \nabla_{p_j} h_{ij}(p(x))^{\top}g_j(x))\Psi} ,
\end{align}
where $h_{ij}(x) \mapsto \gamma h_{ij}(x)^{3}$ is the extended class-$\mathcal{K}$ function, for some $\gamma > 0$.

Now, it remains to formulate \eqref{eq:barrier-certificate} as in a form conducive to an optimization program (i.e., as an inequality constraint).  It is then convenient to define a matrix-valued function $A$ and a vector-valued function $b$ so that the constraint can take the ensemble form $A(x)u \geq b(x)$.  We do so by defining the following row-vector $A_{(i,j)}(x)$ and scalar $b_{(i,j)}(x)$ for each pair $(i,j)$ of robots
\begin{align}
    & A_{(i,j)}(x) = 
    \begin{bmatrix}
        a^{(i,j)}_{1} & \ldots & a^{(i,j)}_{N}
    \end{bmatrix} \\
    & b_{(i,j)}(x) = - \gamma h_{ij}^3(x) \\
    & - \min{(\nabla_{p_i} h_{ij}(p(x))^{\top}g_i(x) + \nabla_{p_j} h_{ij}(p(x))^{\top}g_j(x))\Psi}
\end{align}
where $a^{(i,j)}_{k} \in \mathbb{R}^{1 \times 2}$ is defined w.r.t. each $A_{(i,j)}(x)$ as
\begin{equation}
    a^{(i,j)}_{k} = 
    \begin{cases}
    \nabla_{p_i} h_{ij}(p(x))^{\top}g_i(x), \text{ if } k = i \\
    \nabla_{p_j} h_{ij}(p(x))^{\top}g_j(x), \text{ if } k = j \\
    0, \text{ o.w.}
    \end{cases}
\end{equation}
Note that this process is repeated $\forall i \in \{1,\ldots,N-1\}$ and $\forall j \in \{i+1,\ldots,N\}$ to account for all possible pairwise-collisions, and the $A$ and $b$ matrices are then obtained by vertically stacking all the generated $A_{(i,j)}$ and $b_{(i,j)}$, respectively.  At last, by using all the variables introduced so far, we can formulate a similar QP to the one in Proposition~\ref{prop:controller-synthesis} as
\begin{align}
\label{eq:QP2}
    u^*=\underset{u \in \mathbb{R}^m}{\text{argmin.}} \quad & ||L_cG_c(u_{\text{nom}}(x) - u)||^2 \\
    \text{s.t.} \quad  & \left \| u \right \|_{\infty}  \leq u_{\text{max}}\\
    & A(x)u \geq b(x) , 
\end{align}
where 
 \begin{align}
    & G_c = I_{N} \otimes G,~ L_c =  I_{N} \otimes L ,
\end{align}
where $\otimes$ denotes the Kronecker product.  This formulation ensures that the barrier functions minimally alter $u_{\text{nom}}$ to render the input safe while taking actuation limits into account. The matrix $G_c$ is introduced to transform the differential-drive inputs to linear and angular velocities, and $L_c$ is a weighting matrix that alleviates dead-lock situations by encouraging alterations in the angular velocities rather than the linear velocities.  To account for actuator limits, \eqref{eq:QP2} includes the constraint $\|u\| \leq u_{\text{max}}$, which does not change results of Proposition~\ref{prop:controller-synthesis}.  In the next section, we present extensive experimentation demonstrating how the robust CBF formulation decreases the number of collisions during autonomous operation of the Robotarium.

\section{Experiments} 
\label{sec:experiments}

To highlight the fact that the work in this paper is a milestone toward achieving autonomy in the Robotarium, we present the results yielded by extended experimentation. This scenario showcases how the robust CBF formulation drastically decreases the frequency of collisions at almost no additional computational cost. 

The setup is as follows. The experiment utilizes $22$ GRITSbots, and their parameterisation is shown in \autoref{tab:parameters}.  In the case of the non-robust CBF experiment, $\psi$ is simply changed to $0$. We run an experiment repeatedly for a given period of time using the non-robust CBF formulation to serve as a base case.  Then, the same experiment is run repeatedly using the robust CBF formulation, and we compare the frequency of collisions in both experiments.  The initialization of the experiment involves driving the $22$ GRITSbots to a circular formation as shown in \autoref{fig:xbarrier}.  Upon the termination of the initialization, each robot is commanded to drive to the opposite end of the circle relative to their current position through the use of a proportional controller. Note that this procedure results in all the robots' paths crossing at the center of the circle as highlighted in \autoref{fig:xbarrier}.
This experiment is then run repeatedly for an extended period for each of the two formulations.

\begin{table}[tb]
    \centering
    \begin{tabular}{|c|c|c|c|c|c|c|}
        \hline
       $l_p$ (m)  & $l_b$ (m) & $r$ (m) & $\delta$ (m) & $\gamma$ & $u_{\max}$ (rad/s) & $\psi$ (rad/s)  \\
       \hline
        $0.03$& $0.105$ & $0.016$& $0.12$& $150$ &  $25$ & $5$\\
       \hline
    \end{tabular}
    \caption{Values of the relevant GRITSbot's and experiments' parameters.}
    \label{tab:parameters}
\end{table}

\begin{table}[tb]
    \centering
    \begin{tabular}{|c|c|c|}
        \hline
        & Robust CBF & Non-Robust CBF \\
        \hline
        \hline
       Avg. WCT (ms) & $4.473$ & $4.048$ \\
       \hline
       Var. of WCTs ($\text{ms}^2$) & $5.793$ &  $3.435$\\
       \hline
       Avg. Freq. (Hz)  & $222$ & $247$ \\
       \hline
       Time Violated (s) & $0$ & $138$ \\
       \hline
    \end{tabular}
    \caption{Comparison of the Wall-Clock Times (WCTs) for solving the Quadratic Programs with and without the robust CBF formulation. The last entry is the duration during which the constraint was violated for each experiment.}
    \label{tab:wct}
\end{table}

\begin{figure*}[tb]
\centering
\begin{minipage}{0.5\linewidth}
   \includegraphics[width=1\textwidth]{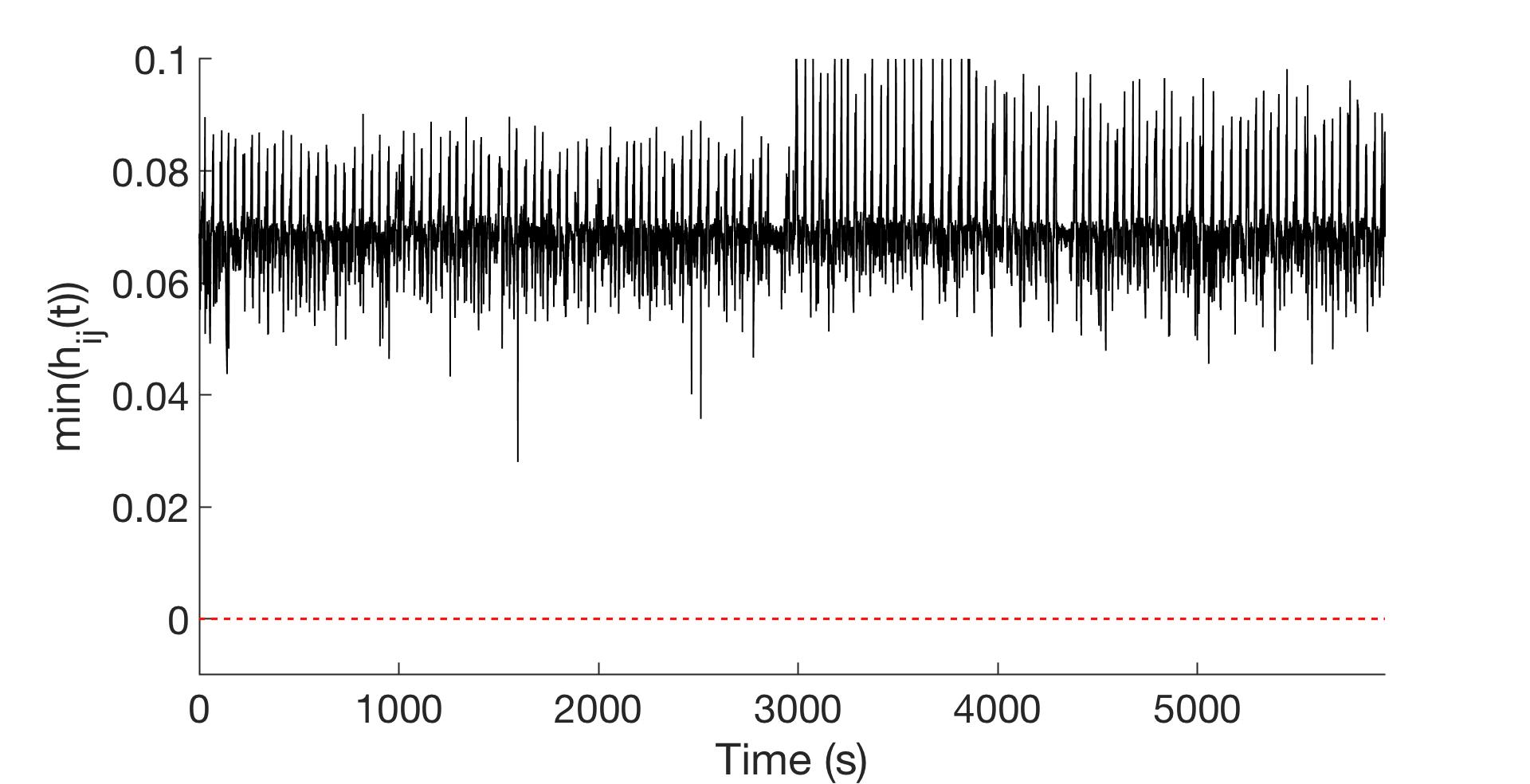}
\end{minipage}~%
\begin{minipage}{0.5\linewidth}
   \includegraphics[width=1\textwidth]{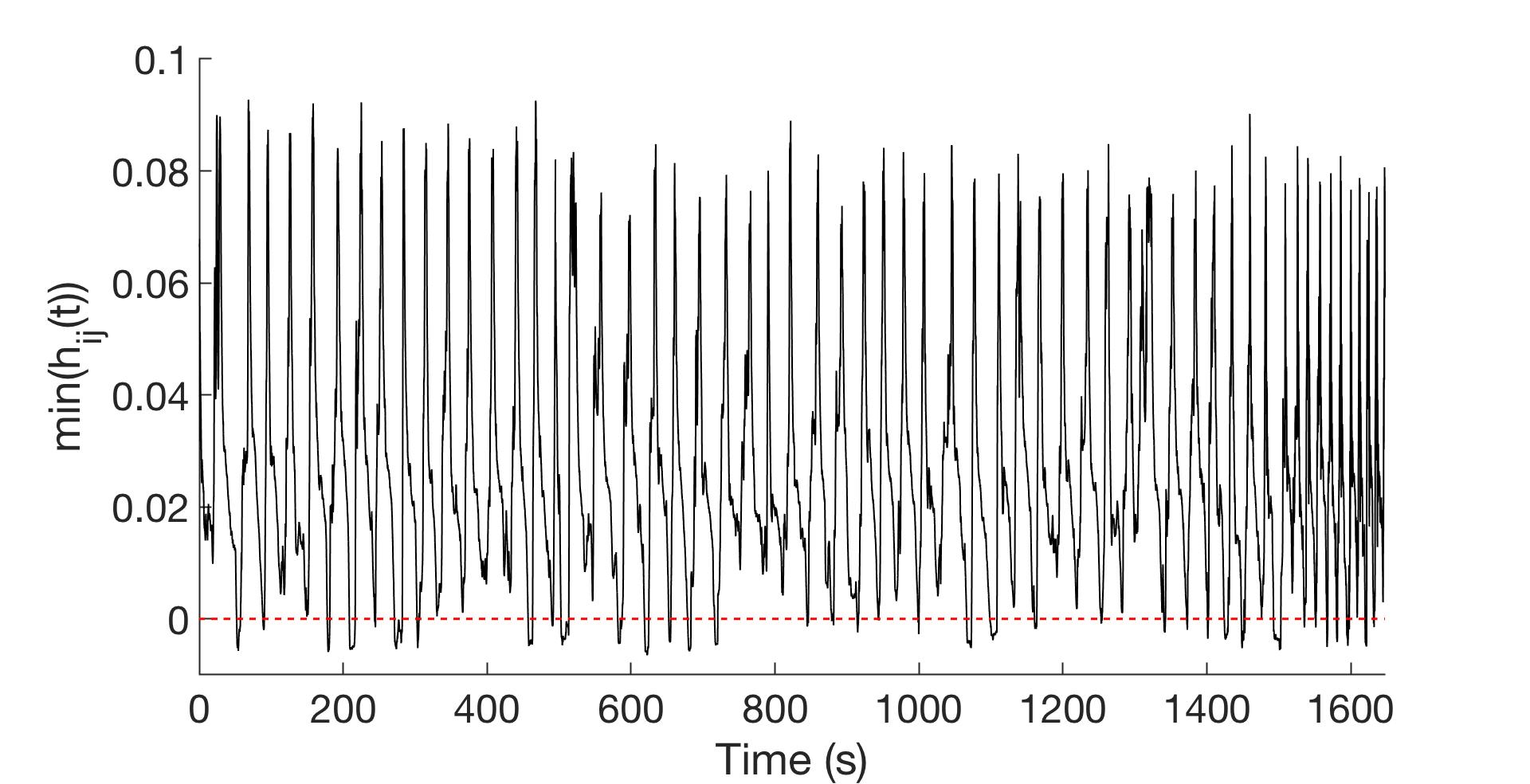}
\end{minipage} 
\caption{Plots of $\min(h_{ij}(t))$ over time for the experiments with and without the robust CBF formulation (left and right, respectively).  At each point in time, if the value $\min(h_{ij}(t))$ is below the dashed line (i.e. value is negative) then the collision-avoidance constraint is violated.  The robust CBF formulation (left) encounters 0 collisions whereas the non-robust CBF formulation (right) periodically encounters collisions as shown in \autoref{tab:wct}.}
\label{fig:minh}
\end{figure*}

The metrics upon which we compare the performance of the non-robust and robust CBF formulations is the frequency of collisions occurring in each experiment and the wall-clock times associated with solving each formulation's QP. Specifically, we record the minimum $h_{ij}$ value at each time step and check if it is negative (i.e. an occurrence of a collision).  The plot of the minimum value of $h_{ij}$ over time is shown in \autoref{fig:minh}. It is clear by inspection that $0$ collisions occur during the robust CBF formulation. On the other hand, \autoref{fig:minh} displays the result of the non-robust formulation where the minimum value of $h_{ij}$ frequently drops below $0$. This violation mainly occurs when the robots are clustered near the center of the circle.

Focusing on the run time, solving the robust CBF QP averaged a wall-clock time of $4.5$~ms, translating to a frequency of approximately $220$~Hz, resulting in only a $25$~Hz decrease compared to solving the non-robust CBF QP.  The reason the frequency decreases slightly is that the only computation added is a $\min$ operation over $p$ values, which is linear with respect to $p$ ($O(p)$), where $p$ is the number of points forming the convex hull of the disturbance.  The comparison of the wall-clock times of both experiments is shown in \autoref{tab:wct}.

This experiment demonstrates how the robust CBF formulation drastically reduces the number of collisions during the Robotarium's operation, which in turn aids in the process of fully automating the Robotarium.  For example, in its current state, a human operator must manually flag an experiment for re-execution if a collision occurs.  The robust CBFs decrease the need for this human operator and increase the throughput of successful experiments via the reduction of collisions.

\section{Conclusion}
\label{sec:conclusion}

This paper extended barrier functions to control systems with an additive disturbance, resulting in robust control barrier functions.  Moreover, this work showed that checking a finite number of points actually accounts for a wide class of disturbances.  As such, these robust control barrier functions remain amenable to controller synthesis via quadratic programming.  Experimental results showcased the efficacy of the approach in a long-term experiment in the Robotarium.

 
\bibliographystyle{IEEEtran}
\bibliography{IEEEabrev,citations}

\begin{thebibliography}{10}
\providecommand{\url}[1]{#1}
\csname url@samestyle\endcsname
\providecommand{\newblock}{\relax}
\providecommand{\bibinfo}[2]{#2}
\providecommand{\BIBentrySTDinterwordspacing}{\spaceskip=0pt\relax}
\providecommand{\BIBentryALTinterwordstretchfactor}{4}
\providecommand{\BIBentryALTinterwordspacing}{\spaceskip=\fontdimen2\font plus
\BIBentryALTinterwordstretchfactor\fontdimen3\font minus
  \fontdimen4\font\relax}
\providecommand{\BIBforeignlanguage}[2]{{%
\expandafter\ifx\csname l@#1\endcsname\relax
\typeout{** WARNING: IEEEtran.bst: No hyphenation pattern has been}%
\typeout{** loaded for the language `#1'. Using the pattern for}%
\typeout{** the default language instead.}%
\else
\language=\csname l@#1\endcsname
\fi
#2}}
\providecommand{\BIBdecl}{\relax}
\BIBdecl

\bibitem{pickem2017robotarium}
D.~{Pickem}, P.~{Glotfelter}, L.~{Wang}, M.~{Mote}, A.~{Ames}, E.~{Feron}, and
  M.~{Egerstedt}, ``The robotarium: A remotely accessible swarm robotics
  research testbed,'' in \emph{2017 IEEE International Conference on Robotics
  and Automation (ICRA)}, 5 2017, pp. 1699--1706.

\bibitem{ames2014}
A.~D. {Ames}, J.~W. {Grizzle}, and P.~{Tabuada}, ``Control barrier function
  based quadratic programs with application to adaptive cruise control,'' in
  \emph{53rd IEEE Conference on Decision and Control}, 12 2014, pp. 6271--6278.

\bibitem{xu2015}
X.~Xu, P.~Tabuada, J.~W. Grizzle, and A.~D. Ames, ``Robustness of control
  barrier functions for safety critical control,'' \emph{IFAC-PapersOnLine},
  vol.~48, no.~27, pp. 54 -- 61, 2015, analysis and Design of Hybrid Systems
  ADHS.

\bibitem{xu2018}
X.~Xu, ``Constrained control of input-output linearizable systems using control
  sharing barrier functions,'' \emph{Automatica}, vol.~87, pp. 195 -- 201,
  2018.

\bibitem{wang2016}
L.~{Wang}, A.~D. {Ames}, and M.~{Egerstedt}, ``Multi-objective compositions for
  collision-free connectivity maintenance in teams of mobile robots,'' in
  \emph{2016 IEEE 55th Conference on Decision and Control (CDC)}, 12 2016, pp.
  2659--2664.

\bibitem{wang2016-2}
L.~{Wang}, A.~{Ames}, and M.~{Egerstedt}, ``Safety barrier certificates for
  heterogeneous multi-robot systems,'' in \emph{2016 American Control
  Conference (ACC)}, 7 2016, pp. 5213--5218.

\bibitem{PaulNBF}
P.~{Glotfelter}, J.~{Cortés}, and M.~{Egerstedt}, ``Nonsmooth barrier
  functions with applications to multi-robot systems,'' \emph{IEEE Control
  Systems Letters}, vol.~1, no.~2, pp. 310--315, Oct 2017.

\bibitem{glotfelter2018}
P.~{Glotfelter}, J.~{Cort\'es}, and M.~{Egerstedt}, ``Boolean composability of
  constraints and control synthesis for multi-robot systems via nonsmooth
  control barrier functions,'' in \emph{2018 IEEE Conference on Control
  Technology and Applications (CCTA)}, 8 2018, pp. 897--902.

\bibitem{rimon1992}
E.~Rimon and D.~E. Koditschek, ``Exact robot navigation using artificial
  potential functions,'' \emph{{IEEE} Trans. Robotics and Autom.}, vol.~8,
  no.~5, pp. 501--518, 10 1992.

\bibitem{nguyen2016optimal}
Q.~Nguyen and K.~Sreenath, ``Optimal robust safety-critical control for dynamic
  robotics,'' \emph{International Journal of Robotics Research (IJRR), in
  review}, 2016.

\bibitem{gurriet2018invariance}
T.~{Gurriet}, A.~{Singletary}, J.~{Reher}, L.~{Ciarletta}, E.~{Feron}, and
  A.~{Ames}, ``Towards a framework for realizable safety critical control
  through active set invariance,'' in \emph{2018 ACM/IEEE 9th International
  Conference on Cyber-Physical Systems (ICCPS)}, April 2018, pp. 98--106.

\bibitem{AmesBarriers}
A.~D. Ames, X.~Xu, J.~W. Grizzle, and P.~Tabuada, ``Control barrier function
  based quadratic programs with application to automotive safety systems,''
  \emph{CoRR}, vol. abs/1609.06408, 2015.

\bibitem{cortes2008}
J.~Cort{\'e}s, ``Discontinuous dynamical systems,'' \emph{IEEE Control Syst.
  Mag.}, vol.~28, no.~3, pp. 36--73, 6 2008.

\bibitem{cortes2008discontinuous}
J.~Cortes, ``Discontinuous dynamical systems,'' \emph{IEEE Control systems
  magazine}, vol.~28, no.~3, pp. 36--73, 2008.

\bibitem{clarke1990}
F.~Clarke, \emph{Optimization and Nonsmooth Analysis}.\hskip 1em plus 0.5em
  minus 0.4em\relax Soc. for Ind. and Appl. Math., 1990.

\bibitem{Hatanaka2009}
Y.~{Igarashi}, T.~{Hatanaka}, M.~{Fujita}, and M.~W. {Spong}, ``Passivity-based
  attitude synchronization in$se(3)$,'' \emph{IEEE Transactions on Control
  Systems Technology}, vol.~17, no.~5, pp. 1119--1134, Sep. 2009.

\bibitem{Park2016}
H.~Park and S.~Hutchinson, ``An efficient algorithm for fault-tolerant
  rendezvous of multi-robot systems with controllable sensing range,'' in
  \emph{2016 IEEE International Conference on Robotics and Automation (ICRA)},
  5 2016, pp. 358--365.

\bibitem{morris2013}
B.~{Morris}, M.~J. {Powell}, and A.~D. {Ames}, ``Sufficient conditions for the
  lipschitz continuity of qp-based multi-objective control of humanoid
  robots,'' in \emph{52nd IEEE Conference on Decision and Control}, 12 2013,
  pp. 2920--2926.

\end{thebibliography}

\end{document}